\documentclass[twoside]{article}

\usepackage{aistats2020}
\usepackage{xcolor}

\usepackage[utf8]{inputenc}
\usepackage[margin=1in]{geometry}
\usepackage{amssymb,amsmath,amsthm}
\usepackage{verbatim,hyperref}
\usepackage{dsfont}
\usepackage{algorithm}
\usepackage[noend]{algpseudocode}
\usepackage{thmtools,thm-restate}

\usepackage{thmtools}
\usepackage{mathtools}
\usepackage{enumerate}

\newcommand{\R}{\mathbb{R}}

\newtheorem{theorem}{Theorem}
\newtheorem{lemma}{Lemma}

\newtheorem{corollary}{Corollary}

\newtheorem{fact}{Fact}
\newtheorem{remark}{Remark}

\theoremstyle{definition}
\newtheorem{definition}{Definition}

\newcommand{\e}{\epsilon}

\newcommand{\poly}{\text{poly}}
\newcommand{\supp}{\text{supp}}

\newcommand{\mathbbm}[1]{\text{\usefont{U}{bbm}{m}{n}#1}}
\DeclareMathOperator*{\argmax}{arg\,max}

\newtheorem{innercustomdef}{Definition}
\newenvironment{customdef}[1]
  {\renewcommand\theinnercustomdef{#1}\innercustomdef}
  {\endinnercustomdef}
\usepackage[round]{natbib}

\begin{document}


%

%

\twocolumn[

\aistatstitle{Constructing a provably adversarially-robust classifier from a high accuracy one}

\aistatsauthor{ Grzegorz Głuch \And Rüdiger Urbanke }

\aistatsaddress{ EPFL \And EPFL }
]

\begin{abstract}
Modern machine learning models with very high accuracy have been shown to be vulnerable to small, adversarially chosen perturbations of the input. Given black-box access to a high-accuracy classifier $f$, we show how to construct a new classifier $g$ that has high accuracy and is also robust to adversarial $\ell_2$-bounded perturbations. Our algorithm builds upon the framework of \textit{randomized smoothing} that has been recently shown to outperform all previous defenses against $\ell_2$-bounded adversaries. Using techniques like random partitions and doubling dimension, we are able to bound the adversarial error of $g$ in terms of the optimum error. In this paper we focus on our conceptual contribution, but we do present two examples to illustrate our framework. We will argue that, under some assumptions, our bounds are optimal for these cases.
\end{abstract}


\section{INTRODUCTION}


Modern neural networks achieve high accuracy on tasks such as image classification (\cite{howtobecomefamous}) or speech recognition (\cite{speachrecognition}) but have been shown to be susceptible to small, adversarially-chosen perturbations of the inputs (\cite{intriguingprop}, \cite{neuralnetseasilyfooled}, \cite{evassionattackontesttime}): given an input $x$, which is correctly classified by a neural network, one is often able to find a small perturbation $\delta$ such that $x + \delta$ is misclassified by the neural network, whereas $x$ and $x+\delta$ are virtually indistinguishable to the human eye. 

Many empirical approaches have been proposed for building ``robust'' classifiers. One of the most successful ones is the framework of \textit{adversarial training} (\cite{adversarialtrainginggoodfellow}, \cite{adversarialtraingingkurakin}, \cite{adversarialtraingngmadry}). Unfortunately these techniques usually protect only against restricted types of adversaries. Moreover, many of the heuristic defenses were shown to break in the presence of suitably powerful adversaries (\cite{bypassing}, \cite{obfuscatedGradient}, \cite{breakingDefenses}).

\textit{Certifiable robust} classifiers, on the other hand, are classifiers whose predictions are verifiably constant within a neighborhood of a query point. The first such classifiers were introduced by  \cite{certifiableRobust1} and \cite{certifiableRobust2}. \textit{Randomized smoothing} was considered in (\cite{Lacuyer19}, \cite{secondOrderAttack}, \cite{smoothingPearson} and \cite{smoothingMicrosoft}). This approach works as follows.

Let $f$ be any classifier which maps $\mathbb{R}^d$ to classes $\mathcal{Y}$. The smoothed classifier $g$ classifies an input $x$ as that class $c$ that is most likely to be returned by $f$ on input $x + \delta$, where $\delta \sim \mathcal{N}(0,\sigma^2 I)$.

It was shown in \cite{Lacuyer19} that this approach scales well and one can use it to train certifiably robust classifier for ImageNet. In \cite{smoothingPearson} it is shown that for $\ell_2$ perturbations \textit{randomized smoothing} outperforms other certifiably defenses previously proposed. Moreover the authors show how to derive a robustness radius guarantee for an input $x$. To derive the bound one defines for a class $c \in \mathcal{Y}$ the probability $p_c := \mathbb{P}_\delta(f(x + \delta) = c)$, where the perturbation $\delta$ is chosen according to $\delta \sim \mathcal{N}(0,\sigma^2 I)$. Then one argues that if there exists a $c$ such that $p_c \gg \max_{c' \neq c} p_{c'}$ then the robustness radius at $x$ is big. Unfortunately, even if the base classifier $f$ has very high accuracy we don't know much about the structure of $\{p_c\}_{c \in \mathcal{Y}}$. Thus it's hard to reason about the robustness radii. 
These shortcomings point to the following question:

\textit{Having a black-box access to a high accuracy classifier $f$ is it possible to construct a new classifier $g$ that is guaranteed to be both robust and has high accuracy?} 

Note that robustness without an accuracy constraint is trivially achieved by a constant classifier, and high accuracy without a robustness constraints has also been shown to be achievable in many settings of interest. The real question of interest therefore only appears if we require both types of constraints.

\textbf{Our contributions:} 
We show a framework for transforming \textbf{any} high accuracy classifier $f$ into a provably robust and high accuracy classifier $g$. Moreover we show what the optimal classifier for a given learning task is and then relate the performance of $g$ to this optimum. We present two instances of this framework. In order to keep the exposition simple, we limit our setting to $\ell_2$-robustness. The ideas apply more generally, but the details will differ.

In the first instance we show that if $f$ satisfies a suitable property (similar to a property implicitly assumed in \cite{smoothingPearson} and \cite{smoothingMicrosoft}) then $g$ can be evaluated with only black-box access to $f$.

In the second instance we prove that, without any assumptions on $f$, a robust classifier $g$ can be evaluated if besides $f$ we also have access to an oracle $\mathcal{O}$ that provides  {\em unlabeled} i.i.d. samples from the underlying distribution. Notice that this model is not very restrictive. A similar setting occurs in semi-supervised learning where the learner has access to a dataset of labeled data $D_l$ and also to (an often much larger) dataset $D_u$ of unlabeled samples (see \cite{semisupervisedbook}). In this scenario $D_u$ serves as the oracle $\mathcal{O}$.

Even though our main contribution is a conceptual one, 
we also present two implementations of these methods that achieve different runtime/robustness tradeoffs. In the end we give examples of binary classification tasks (e.g. adversarial spheres from \cite{adversarialSpheres}) and compare the performance of our methods on these tasks to the optimum.

\section{OUR TECHNIQUES}

Let us present an overview of our approach.

\subsection{Randomized smoothing}

Our techniques build upon \textit{randomized smoothing} from \cite{Lacuyer19}, \cite{secondOrderAttack}, \cite{smoothingPearson} and \cite{smoothingMicrosoft}. Consider a classifier $f$ that maps $\mathbb{R}^d$ to classes $\mathcal{Y}$. \textit{Randomized smoothing} is a method that produces a new, \textit{smoothed} classifier $g$. The smoothed classifier $g$ assigns to a query point $x$ the class that is most likely to be returned by $f$ under random Gaussian noise perturbations:
\begin{equation}\label{eq:regularsmoothing}
g(x) := \argmax_{c \in \mathcal{Y}} \mathbb{P}[f(x+\delta) = c], \delta \sim \mathcal{N}(0,\sigma^2 I).
\end{equation}
Note that $g$ can also be expressed as:
\begin{equation}\label{eq:regularsmoothingdifformulation}
g(x) = \argmax_{c \in \mathcal{Y}} \int_{\mathbb{R}^d}\mathbbm{1}_{ \{f(x) = c\} } \gamma(x-z) dz \text{,}
\end{equation}
where $\gamma$ is the density function of $\mathcal{N}(0,\sigma^2 I)$.

Unfortunately it is easy to design a learning task and a classifier $f$ with low standard error such that $g$, computed according to \eqref{eq:regularsmoothing}, has high error. For instance, imagine the following binary classification task in $\mathbb{R}^2$. 
We generate $x \in \mathbb{R}^2$ uniformly at random from a union of two discs $B_{-},B_{+}$ of radius $1$ centered at $(-2,0)$ and $(2,0)$, respectively. We assign the label $y = -1$ if $x$ belongs to $B_{-}$ and the label $y = +1$ otherwise. Let $f(x) = -1$ if $x \in B_{-}$ and $f(x) = +1$ otherwise (i.e., for all points $x \not \in B_{-}$). Observe that $f$ has a standard risk of $0$. If we now compute $g$ according to \eqref{eq:regularsmoothing}, then  $g(x) = +1$ for all $x$ if $\sigma \geq 1/(\sqrt{2}\; \text{InvErfc}(\tfrac12)) \sim 1.4826$. This means that $g$ has an error of $\tfrac12$. 

The reason that we were able to construct such an example is that in \eqref{eq:regularsmoothing} the smoothing is performed independent of the data. A natural idea to fix this is to perform the smoothing ``conditioned" on the data distribution. For instance, in the example above we would like to not take points outside $B_{-} \cup B_{+}$ into account during smoothing. The formal definition of this approach is as follows:
\begin{equation}\label{eq:restrictedsmoothing}
g(x) := \argmax_{c \in \mathcal{Y}} \int_{\mathbb{R}^d}\mathbbm{1}_{ \{f(x) = c\} } \gamma(x-z)  p_X(z) dz,
\end{equation}
where $p_X$ is the density function of the data distribution. Notice the difference between \eqref{eq:restrictedsmoothing} and \eqref{eq:regularsmoothingdifformulation}. Unfortunately we can construct ``counter examples" even for this modification as shown in the next section.

\subsection{Hard distribution for randomized smoothing described in \eqref{eq:restrictedsmoothing}}

It is possible to create a separable, binary classification task on $\mathbb{R}^d$ and a classifier $f$ such that the standard error of $f$ is $e^{-\Theta(d)}$ but the error of the smoothed classifier $g$ is $\Theta(1)$ (see Appendix A 
for details). That is, the standard error grows by a factor $e^{\Theta(d)}$ when we perform smoothing!  

This example shows that when we use \textit{randomized smoothing} then already the {\em standard} error can grow by a factor exponential in the dimension of the ambient space. As we aim for creating $g$ that is robust \textbf{and} has small error we must try a different approach.

\subsection{Partitions}

The intuitive reason why we were able to construct the example in the previous section is that in \textit{randomized smoothing} it might happen that $1$ misclassified point of $f$ contributes to $e^{\Theta(d)}$ misclassified points of $g$. To prevent that we use space partitions.

Assume that in a binary classification task the distance between the two classes is at least $\e$. Assume further that we partition $\mathbb{R}^d$ into sets $S_1, S_2, \dots$, each of diameter at most $\e$. Now for $x \in \mathbb{R}^d$  we define $g(x)$ as the class that is most likely returned by $f$ {\bf on points sampled from the data distribution conditioned on being in set $S_i$ to which $x$ belongs}. As the classes are at least $\e$ away from each other and the diameters of the sets in the partition are at most $\e$, each misclassified point of $f$ contributes to at most $2$ misclassified points of $g$ (this will be proven in Lemma~\ref{lem:boundriskofg}). This means that the error of $g$ is bounded in terms of the error of $f$.

But we also want $g$ to be robust. Intuitively we want a big fraction of points to be far from the boundaries of sets $S_1,S_2, \dots$. To do that we use \textit{padded random partitions}, which previously found applications in low distortion embeddings (\cite{spacePartitioning}), locality sensitive hashing (\cite{aproximatenearestneighborinhighdim}) and even spectral algorithms (\cite{multiwaycheeger}). Definitions of random and padded partitions are presented in Section~\ref{sec:randompartitions}.

\subsection{Doubling dimension}

Some random partitions suffer from the big dimension of the space $\mathbb{R}^d$. To improve the guarantees for binary classification tasks that have data lying on lower dimensional manifolds we resort to the notion of \textit{doubling dimension} (Section~\ref{sec:doublingdimension}). This definition captures the intuition that it should be easier to ``describe" a manifold that is lower dimensional.

\subsection{Examples}

In Section~\ref{sec:experiments} we give two examples to analyze the tightness of the bounds obtained in Section~\ref{sec:mainresults}. The first one is a data distribution from \cite{adversarialSpheres}. For this example we show that our approach is competitive against a certain class of classifiers (see Section~\ref{subsec:advspheres} for an in-depth discussion). 
The second example is a data distribution supported on two low dimensional manifolds embedded in high-dimensional space for which we show optimality of our method up to constant factors. 

\section{PRELIMINARIES}\label{sec:preliminaries}

For a distribution $\mathcal{D}$ over $\mathbb{R}^d$ and for a set $A \subseteq \mathbb{R}^d$ let 
$\mu(A) := \mathbb{P}_X(X \in A) \text{.}$ For us, ${\mathcal D}$ will denote the distribution of the data.  For simplicity in this section and the rest of the paper we consider only {\em separable} binary classification tasks. Such tasks are fully specified by ${\mathcal D}$ as well as a {\em ground truth} $h : \R^d \rightarrow \{-1, 1\}$. We note however that one can generalize the results to any binary classification task (see Appendix B 
for a generalization of the definitions from this section). For $x \in \mathbb{R}^d$ and $\epsilon>0$ we write $B_{\e}(x)$ to denote the {\em open ball} with center $x$ and radius $\e$. Most of the proofs are deferred to the Appendix D.

\begin{definition}(\textbf{Risk})
Consider a binary classification task for separable classes with a ground truth $h : 
\mathbb{R}^d \xrightarrow{} \{-1, 1\}$. For a classifier $f : \mathbb{R}^d \xrightarrow{} \{-1, 1\}$ we define the \textbf{R}isk as  
$$R(f) := \mathbb{P}_X (f(X) \neq h(X)) \text{.}$$
\end{definition}

\begin{definition}(\textbf{Adversarial Risk}) \label{def:adversarialrisk}
Consider a binary classification task for separable classes with a ground truth $h : 
\mathbb{R}^d \xrightarrow{} \{-1, 1\}$. For a classifier $f : \mathbb{R}^d \xrightarrow{} \{-1, 1\}$ and $\epsilon \in \mathbb{R}_{\geq 0}$ we define the \textbf{A}dversarial \textbf{R}isk as
$$AR(f,\epsilon) := \mathbb{P}_X (\exists \ \eta \in B_\epsilon \ f(X + \eta) \neq h(X)) \text{.}$$
We also introduce the notation:
$$AR(\epsilon) := \inf_{f} AR(f,\epsilon) \text{.}$$
to denote the smallest achievable adversarial risk for that classification task with a given $\epsilon$.
\end{definition}

\begin{fact}(\textbf{R versus AR}) \label{fact:basicsaboutrisks}
\vspace{-0.3cm}
\begin{itemize}
    \item $AR(f,0) = R(f)$,
    \item $AR(f,\epsilon)$ and $AR(\epsilon)$ are nondecreasing functions of $\epsilon$; combined with the previous point this in particular implies that for $\epsilon \in \mathbb{R}_{\geq 0}$, $AR(f,\epsilon) \geq R(f)$,
    \item $AR(f,\e) \!\leq \!R(f) \!+\! \mathbb{P}_{X \sim \mathcal{D}}[f \text{ $\neg$ const. on } B_\e(X)]$.
\end{itemize}
\end{fact}

\begin{definition}[\textbf{Separation function}]
For a binary classification task for separable classes with a ground truth $h : 
\mathbb{R}^d \xrightarrow{} \{-1, 1\}$ we define a separation function as follows:
$$S(\epsilon) := \inf_{E \subseteq \mathbb{R}^d, d(M_{-} \setminus E, M_{+} \setminus E) \geq \epsilon} \left[ \mathbb{P}_X(X \in E) \right] \text{.}$$
Here $M_{-} = h^{-1}(\{-1\}),M_{+} = h^{-1}(\{1\})$. For a given $\epsilon>0$ this function returns the probability mass that needs to be removed so that the classes are separated by an $\epsilon$-margin.
\end{definition}

\begin{restatable}{lemma}{bestclassifier}
\label{lem:bestclassifier}
For all separable binary classification tasks and all $\epsilon \in \mathbb{R}_{\geq 0}$ we have that:
$$AR(\epsilon) = S(2\epsilon) \text{.}$$
\end{restatable}

{\color{blue}

}

\section{DOUBLING DIMENSION}\label{sec:doublingdimension}

\begin{definition}(\textbf{$\epsilon$-Net})
Let $(M,d)$ be a metric space. For $N \subseteq M$ we say that $N$ is an $\e$-net of $M$ if it satisfies:
\begin{itemize}
    \item For every $u,w \in N$ if $u \neq w$ then $d(u,w) \geq \e$,
    \item $M \subseteq \bigcup_{u \in N} B_{\e}(u)$.
\end{itemize}
\end{definition}

\begin{definition}[\textbf{$\e$-Doubling dimension}]
For a metric space $(M,d)$, let $\lambda$ be the smallest value such that every ball of radius at most $\e$ in $M$ can be covered by $\lambda$ balls of half the radius. We define the \textbf{$\e$-doubling dimension} of $M$ as $dd((M,d),\e) := \log_2 \lambda$.
Sometimes we will omit specifying the metric and write $dd(M,\e)$ when the metric is clear from the context.
\end{definition}

\begin{definition}[\textbf{Doubling dimension}]
For a metric space $(M,d)$ it's doubling dimension is defined as:
$$dd((M,d)) := \sup_{\e > 0} dd((M,d),\e)$$
\end{definition}

\begin{restatable}{fact}{doublingdimofrd}
\label{fact:ddofrd}
$dd((\mathbb{R}^d, \ell_2)) \leq 3d$
\end{restatable}



The next fact was implicitly proven in \cite{curvatureBound}.

\begin{fact}\label{fact:aboutcurvature}
Let $M \subseteq \mathbb{R}^d$ be a $d'$ dimensional manifold such that the second fundamental form is uniformly bounded by $\kappa$. 
Pick $\e \leq 1/2\kappa$. If for all $e' \leq \e$, for all $x \in M$ we have that $B_{\e}(x) \cap M$ has at most $2^{O(d')}$ connected components then $$dd(M,\e) = O(d') \text{,}$$
where the metric on $M$ is the inherited $\ell_2$ metric from $\mathbb{R}^d$.
\end{fact}

\begin{restatable}{lemma}{notmanypoints}
\label{lem:notmanypoints}
Let $(M,d)$ be a metric space with $\e$-doubling dimension $dd$. If all pairwise distances in $N \subseteq M$ are at least $r$ then for any point $x \in M$ and radius $r \leq t \leq \e$ we have $|B_t(x) \cap N| \leq 2^{dd \lceil \log \frac{2t}{r} \rceil}$.
\end{restatable}



\begin{remark}
In the remainder of the paper we will only consider subsets of $\mathbb{R}^d$ and the metric we use is always the inherited $\ell_2$ metric from the whole space.
\end{remark}

\section{RANDOM PARTITIONS}\label{sec:randompartitions}

We now discuss {\em random partitions}, the main technical tool of the paper. For a metric space $(M,d)$ a {\em partition} $\pi$ of $M$ is as a function $\pi : M \xrightarrow{} 2^M$, mapping a point $x \in M$ to the unique set $\pi(x)$ in $\pi$ that contains $x$. 

Although in this section we formulate all statements with respect to a generic $M$, in the sequel it will be important that $M$ equals the support of the data distribution, i.e., $M = \text{supp}({\mathcal D})$. In particular this will come into play when the data lies on a manifold of small dimension embedded in the ambient space. \textbf{To simplify our notation we will not repeat this assertion in each subsequent statement.}

For $\epsilon > 0$ we say that $\pi$ is {\em $\epsilon$-bounded} if $\text{diam}(\pi(x)) \leq \epsilon$ for all $x \in M$. The main object of interest will be {\em random} partitions. We denote a random partition by $\Pi$ and assume that it has distribution ${\mathcal P}$. We say that $\Pi$ is $\epsilon$-bounded if $\Pi$, drawn according to $\mathcal{P}$, is $\epsilon$-bounded with probability $1$.


\begin{definition}[\textbf{Padded partitions}]
For a metric space $(M,d)$ we say that a random partition $\Pi \sim \mathcal{P}$ is $(\epsilon, \beta, \delta)$-\textit{padded} if it is $\epsilon$-bounded and for every $x \in M$:
$$\mathbb{P}_{\Pi \sim {\mathcal P}}[B_{\e/\beta}(x) \not\subseteq \Pi(x)] \leq \delta \text{.}$$
\end{definition}

\begin{restatable}{corollary}{propertyofpadded}
\label{cor:propertyofpadded}
Let $\Pi \sim \mathcal{P}$ be an $(\e,\beta,\delta)$-padded random partition of a metric space $(M,d)$. Then for every distribution $\mathcal{D}$ we have that:
$$\mathbb{E}_{\Pi \sim \mathcal{P}}[\mathbb{P}_{X \sim \mathcal{D}}[B_{\e/\beta}(X) \not\subseteq \Pi(X)]] \leq \delta \text{.}$$
\end{restatable}

Now let's consider two random partitions:
\begin{definition}[\textbf{Cube partition}]\label{def:cubespart}
For the space $(\mathbb{R}^d,\ell_2)$ and parameter $\epsilon$ we define a Cube partition as a partition of $\mathbb{R}^d$ into cubes of width $\epsilon/\sqrt{d}$ corresponding to the shifted lattice $v + \frac{\epsilon}{\sqrt{d}}\cdot \mathbb{Z}^d$. Here the shift $v \sim U([0,\frac{\epsilon}{\sqrt{d}}]^d)$, i.e.,  $v$ is drawn uniformly at random from a fundamental region of the lattice $\frac{\epsilon}{\sqrt{d}}\cdot \mathbb{Z}^d$. A point $x$ which lies in the intersection of two or more cubes is assigned to the one that is crossed first by a ray $x + \alpha(1,1,\dots,1), \alpha \in \mathbb{R}_{\geq 0}$.
\end{definition}

\begin{definition}[\textbf{Ball carving partition}]\label{def:ballspart}
For a bounded $M \subseteq \mathbb{R}^d$ and $\e>0$ we define a ball carving partition as follows. Let $N$ be an $\epsilon/4$-net of $M$. Pick $R$ uniformly at random from the interval $(\epsilon/4, \epsilon/2]$. Let $\sigma$ be a random permutation of $N$. Then for each $u \in N$ define
$$\hat{\Pi}(u) := B_{R}(u) \setminus  \bigcup_{w: \sigma(w) < \sigma(u)} B_{R}(w) \text{.}$$
Since the radius $R$ can be strictly larger than some pairwise distances it can happen that for some $u \in N$, $\hat{\Pi}(u)$ does not contain $u$ itself, leading to a potential inconsistency in our notation for the points of the net $N$. Hence, for all $x \in M$ (and in particular the points of the net $N$ itself) let us define $\Pi(x)$ to be the unique $\hat{\Pi}(w)$, $w \in N$, that contains $x$.
\end{definition}

\begin{restatable}{lemma}{cubepartispadded}
\label{lem:cubepartispadded}
Let $\Pi$ be a Cube partition with parameter $\epsilon$. Then for every $\beta>2\sqrt{d}$ it is $\left(\epsilon,\beta,\frac{O(d^{1.5})}{\beta} \right)$-\textit{padded}.
\end{restatable}

The proof of the following Lemma is a slight modification of a proof presented in (\cite{spacePartitioning}).

\begin{lemma}\label{lem:ballcarvingispadded}
Let $\Pi$ be a Ball carving partition of a bounded $M \subseteq \mathbb{R}^d$ with parameter $\epsilon$. Then for every $\beta>1$ it is $(\epsilon,\beta,\frac{O(dd(M,\e))}{\beta})$-\textit{padded}.
\end{lemma}

\begin{proof}
Recall that the net $N$ underlying the ball carving partition is an $\e/4$-net of $M$. Fix a point $x \in M$ and some $t \in [0,\epsilon/4]$. Let $W = B_{\epsilon/2 + t}(x) \cap N$, and note that by Lemma~\ref{lem:notmanypoints} we have that $m = |W| \leq 6^{dd(M,\e)}$.
Arrange the points $w_1,\dots,w_n \in W$ in order of increasing distance from $x$, and let $I_k$ be the interval $[d(x,w_k) - t, d(x,w_k) + t]$. Let us say that $B_{t}(x)$ is \textit{cut} by a cluster $\hat{\Pi}(w_k)$ if $\hat{\Pi}(w_k) \cap B_t(x) \neq \emptyset$ and $B_t(x) \not\subseteq \hat{\Pi}(w_k)$. Finally, write $\mathcal{E}_k$ for the event that $w_k$ is the minimal element in $W$ (according to $\sigma$) for which 
$\hat{\Pi}(w_k)$ cuts $B_t(x)$. Then, 
\begin{align*}
\mathbb{P}[B_t(x) \text{ is cut}] 
&\leq \sum_{k=1}^m \mathbb{P}[\mathcal{E}_k] \\
&= \sum_{k=1}^m \mathbb{P}[R \in I_k] \cdot \mathbb{P}[\mathcal{E}_k | R \in I_k] \\
&\leq \sum_{k=1}^m \frac{4t}{\epsilon} \cdot \frac{1}{k} \leq \frac{4t}{\epsilon}(1 + \ln m) \text{.}
\end{align*}
Using the fact that $m = |W| \leq 6^{dd(M,\e)}$ we get that:
$$\mathbb{P}[B_t(x) \text{ is cut}] \leq \frac{t \cdot (8 \cdot dd(M,\e) + 4)}{\epsilon} \text{.}$$
\end{proof}

\begin{corollary}\label{cor:ballspaddedforrd}
If $\Pi$ is a Ball carving partition of a bounded $M \subseteq \mathbb{R}^d$ with parameter $\epsilon$ then for every $\beta>1$ it is $\left(\epsilon,\beta,\frac{O(d)}{\beta}\right)$-\textit{padded}.
\end{corollary}

\begin{proof}
It's a consequence of Fact~\ref{fact:ddofrd} and Lemma~\ref{lem:ballcarvingispadded}.
\end{proof}

\section{FROM A PARTITION TO A CLASSIFIER}\label{sec:partitions}

To create a robust classifier $g$ from a low-risk classifier $f$ we will use the following framework:

\begin{algorithm}[H]
\caption{\textsc{Smooth}($f,\mathcal{P}$)}
\label{alg:framework}
\begin{algorithmic}[1]
	\State Partition ``the space'' using $\Pi \sim \mathcal{P}$ 
	\State \Return $g(x) := \text{sgn}(\mathbb{E}_{Z \sim \mathcal{D}}[f(Z) | Z \in \Pi(x) ])$
\end{algorithmic}
\end{algorithm}

First we want to argue that if a partition $\pi$ is $\e$-bounded then $g$ defined in Algorithm~\ref{alg:framework} will have small Risk.

\begin{restatable}{lemma}{boundriskofg}
\label{lem:boundriskofg}
Let $\pi$ be an $\e$-bounded partition. For a given $f$ let $g(x) = \text{sgn}(\mathbb{E}_{Z \sim \mathcal{D}}[f(Z) | Z \in \pi(x) ])$. Then
$$R(g) \leq 2S(\e) + 2R(f) \text{.}$$
\end{restatable}

The following lemma collects the results from previous sections to obtain a bound on the Adversarial Risk of the classifier $g$ in terms of the best possible classifier.

\begin{restatable}{lemma}{combineeverything}
\label{lem:combineeverything}
For all $\e > 0$ and any binary classification task with underlying distribution $\mathcal{D}$ if there exists an $(\e\beta,\beta,\delta)$-padded random partition $\Pi$ of $\text{supp}(\mathcal{D})$ then the following conditions hold. There exists a randomized algorithm ALG that given black-box access to classifier $f$ produces a classifier $g$ such that in expectation over the random choices of ALG:
$$AR(g,\epsilon) \leq 2S(\e\beta) + 2R(f) + \delta$$
and if $AR(\e) > 0$ then:
$$AR(g,\epsilon) \leq \frac{2S(\e\beta)}{S(2\e)} AR(\e) + 2R(f) + \delta \text{.}$$
\end{restatable}

\section{MAIN RESULTS}\label{sec:mainresults}
In this section we use the partitions defined in Section~\ref{sec:partitions} to derive explicit bounds for the Adversarial Risk of the created classifier.

\begin{theorem}\label{thm:cubes}
Assume that Algorithm~\ref{alg:framework} uses Cube partitions (see Definition~\ref{def:cubespart}). Let $\alpha>0$ and $\e > 0$. Then, in expectation over the randomness of the algorithm, 
$$AR(g,\epsilon) \leq 2S \left(\frac{ d^{\frac32} \cdot \e}{\alpha} \right) + 2R(f) + O(\alpha)$$
and if $AR(\e) > 0$ then
$$AR(g,\epsilon) \leq \frac{2S \left(\frac{ d^{\frac32} \cdot \e}{\alpha}\right)}{S(2\e)} AR(\e) + 2R(f) + O(\alpha).$$
\end{theorem}

\begin{proof}
It is a consequence of Lemma~\ref{lem:combineeverything} and Lemma~\ref{lem:cubepartispadded}.
\end{proof}

To understand the interplay of the parameters it's instructive to consider the following case. If $S \left(\frac{O\left(d^{\frac32}\right)}{ \alpha} \epsilon \right)$ and $S(2\epsilon)$ are comparable, say their ratio is upper-bounded by a constant $C$, and $\alpha$ is some small constant then the theorem says that the classifier produced by the algorithm satisfies:
\begin{equation}
AR(g,\epsilon) \leq 2C \cdot AR(\epsilon) + 2 \cdot R(f) + O(1) \text{.}
\end{equation}
That is, the produced classifier is at most $2C$ times (plus additive error) worse than the optimal one.

Next we present an algorithm with a better bound that uses Ball carving partitions.

\begin{theorem}\label{thm:balls}
Assume that Algorithm~\ref{alg:framework} uses Ball carving partitions (see Definition~\ref{def:ballspart}). Let $\alpha>0$ and $\e > 0$. Then, in expectation over the randomness of the algorithm,
$$AR(g,\epsilon) \leq 2S \left(\frac{ d \cdot \e}{\alpha} \right) + 2R(f) + O(\alpha)$$
and if $AR(\e) > 0$ then:
$$AR(g,\epsilon) \leq \frac{2S \left(\frac{ d \cdot \e}{\alpha}\right)}{S(2\e)} AR(\e) + 2R(f) + O(\alpha).$$
\end{theorem}

\begin{proof}
It is a consequence of Lemma~\ref{lem:combineeverything} and Corollary~\ref{cor:ballspaddedforrd}.
\end{proof}

Finally we generalize Theorem~\ref{thm:balls} to the case when the support of the underlying distribution is a low-dimensional manifold.

\begin{theorem}\label{thm:manifold}
Assume that Algorithm~\ref{alg:framework} uses Ball carving partitions (see Definition~\ref{def:ballspart}) and that $\text{supp}(\mathcal{D}) \subseteq \mathbb{R}^d$ is a $\textbf{d'} \leq d$ dimensional manifold such that the second fundamental form is uniformly bounded by $\kappa$. Assume further that for all $x \in M$, and for all $r \leq 1/2\kappa$ the intersection $B_{r}(x) \cap M$ has at most $2^{O(d')}$ connected components. Let $\alpha>0$ and $\e \leq \frac{\alpha}{O(d') \kappa}$. Then, in expectation over the randomness of the algorithm,
$$AR(g,\epsilon) \leq 2S \left(\frac{ \textbf{d'} \cdot \e}{\alpha} \right) + 2R(f) + O(\alpha)$$
and if $AR(\e) > 0$ then:
$$AR(g,\epsilon) \leq \frac{2S \left(\frac{ \textbf{d'} \cdot \e}{\alpha}\right)}{S(2\e)} AR(\e) + 2R(f) + O(\alpha).$$


\end{theorem}

\begin{proof}
It's a consequence of Lemma~\ref{lem:combineeverything}, Lemma~\ref{lem:ballcarvingispadded} and Fact~\ref{fact:aboutcurvature}.
\end{proof}

Note that all theorems in this section give bounds in the expectation over the randomness of the algorithms. By applying Markov inequality, we can convert these bounds to bounds that are worse by a factor $\gamma>1$ but hold  with probability $1 - 1/\gamma$.

\section{COMPUTING $\text{sgn}(\mathbb{E}[f(Z) | Z \!\!\in\!\! \pi(x) ])$}
Recall that $g(x) := \text{sgn}(\mathbb{E}_{Z \sim \mathcal{D}}[f(Z) | Z \in \pi(x) ])$.
As we do not know the distribution $\mathcal{D}$ we cannot compute this expectation directly.
\subsection{Scheme A: Approximation with oracle}
One approach is to approximate the expectation by a sample mean
\begin{equation}\label{eq:approxtogdata}
\hat{g}(x) := \text{sgn} \left(\frac{1}{s} \sum_{i=1}^s f(Z_i) \right)  \text{,}
\end{equation}
where the $Z_i$'s are i.i.d. samples from the distribution $\mathcal{D}$ conditioned on being inside $\pi(x)$. To compute this sum we need samples from $\mathcal{D}$. Note that {\em unlabeled} samples suffice. 

We will bound the number of samples needed to estimate $\hat{g}$ so that $\hat{g}$ has small adversarial risk. Let $x \in \R^d$, and assume that $| \mathbb{E}_{Z \sim \mathcal{D}}[f(Z) | Z \!\!\in\!\! \pi(x) ] - \frac{1}{2} | \geq 0.1$. If we use $s$ samples to estimate $\hat{g}(x)$ according to \eqref{eq:approxtogdata}, then, using standard tail bounds,
\begin{equation}\label{eq:succesprob}
\mathbb{P}[g(x) \neq \hat{g}(x)] \leq e^{-\Theta(s)} \text{.}
\end{equation}
Now assume that $\pi$ has $Q$ sets $S_1,S_2,\dots,S_Q$. For $i \in \{1,\dots,Q\}$ let $p_i := \mathbb{P}_{X \sim \mathcal{D}}[X \in S_i]$.
We only need to worry about sets $S_i$ whose probability $p_i$ is not too small. Hence, let $H \subseteq \{i \in \{1,\dots,Q\} :  p_i \geq \frac{R(f)}{Q}\}$. We can argue now, as in the coupon collector's problem, that if we draw \begin{align}
O\left(\frac{Q}{R(f)} \log \left(\frac{Q}{R(f)} \right)+\frac{Q \log(Q)}{R(f)} \log \log \left(\frac{Q}{R(f)} \right)\right) \label{equ:numofsamples}
\end{align}
samples from $\mathcal{D}$ then with constant probability, for every $i \in H$ at least $\Theta \left(\log(Q) \right)$ samples will end up in set $S_i$. 

Observe that sets not in $H$ cover negligible mass of $\mathcal{D}$:
\begin{equation}\label{eq:smallsets}
\sum_{i \in \{1,\dots,Q \} \setminus H} p_i \leq R(f) \text{.}
\end{equation}
Now let $F := \{i \in \{1,\dots,Q\} : | \mathbb{E}_{Z \sim \mathcal{D}}[f(Z) | Z \!\!\in\!\! S_i ] - \frac{1}{2} | \leq 0.1 \}$ and notice that sets from $F$ also cover negligible mass of $\mathcal{D}$:
\begin{equation}\label{eq:closesets}
\sum_{i \in F} p_i \leq O(R(f)) \text{,}
\end{equation}
because if $i \in F$ then at least a $0.4$ fraction of points from $S_i$ is misclassified. Putting everything together: by \eqref{eq:succesprob} and the union bound over $Q$ sets, if we sample \eqref{equ:numofsamples} points from $\mathcal{D}$ then with constant probability, for every $i \in \{1,\dots,Q\} \setminus (F \cup H)$ $\hat{g}$ is equal to $g$ on $S_i$, which by using \eqref{eq:smallsets}, \eqref{eq:closesets} and Lemma~\ref{lem:boundriskofg} implies that
$R(\hat{g}) \leq O(R(f) + S(\e))$. 
As a consequence, all theorems from Sections~\ref{sec:mainresults} remain true in this setting up to some changes in the constant factors. For instance a variant of Theorem~\ref{thm:balls} would state: 

\begin{theorem}\label{thm:estimatorofg}
Assume that we sample $$O\left(\frac{Q}{R(f)} \log \left(\frac{Q}{R(f)} \right)+\frac{Q \log(Q)}{R(f)} \log \log \left(\frac{Q}{R(f)} \right)\right)$$ points from $\mathcal{D}$ to estimate $\hat{g}$.
Assume further that Algorithm~\ref{alg:framework} uses Ball carving partitions (see Definition~\ref{def:ballspart}). Let $\alpha>0$ and $\e > 0$. Then, with constant probability over the randomness of the algorithm,
$$AR(\hat{g},\epsilon) \leq O\left(S \left(\frac{ d \cdot \e}{\alpha} \right) + R(f) + \alpha \right)$$
and if $AR(\e) > 0$ then:
$$AR(\hat{g},\epsilon) \leq O \left(\frac{S \left(\frac{ d \cdot \e}{\alpha}\right)}{S(2\e)} AR(\e) + R(f) + \alpha \right).$$
\end{theorem}

\subsection{Scheme B: Approximation by uniform sampling}
If $Q$ is large then an alternative approach to estimating $g$ might be preferable. One might hope that 
\begin{equation}\label{eq:approxtog}
g(x) \approx \text{sgn}(\mathbb{E}_{Z \sim U(\pi(x))}[f(Z)]).
\end{equation}
In words, the expectation of $f$ over the whole set $\pi(x)$ is a good proxy to the expectation of $f$ with respect to $\mathcal{D}$ conditioned on being in set $\pi(x)$. If that is the case then instead of performing the smoothing with respect to the data distribution $\mathcal{D}$ we smooth with respect to the uniform distribution on a set of the partition. 
There are experimental results that indicate that assumption \eqref{eq:approxtog} is reasonable. In particular, the approach to approximate $g(x)$ according to \eqref{eq:approxtog} is similar to the smoothing used in \cite{smoothingPearson} and \cite{smoothingMicrosoft} -- in these works the smoothing is performed by adding a random Gaussian noise to the input. So also in this case the smoothing does not depend on $\mathcal{D}$. Authors of these papers show that their methods outperform all previous defenses against $\ell_2$-norm adversarial perturbations. This suggests that assumption \eqref{eq:approxtog} holds. 

A disadvantage of that approach is that it's hard to prove any theoretical guarantees for this algorithm because, as we discussed before, classifiers with small risk can still behave widely outside of $\text{supp}(\mathcal{D})$. The main advantage of this approach is that we don't require any additional data, apart from access to $f$, to compute $\hat{g}$. So if \eqref{eq:approxtog} holds then the theorems from Section~\ref{sec:mainresults} give a direct, affirmative answer to the question posed in the introduction.

The discussion about running times is deferred to the Appendix C. 

\section{THOUGHT EXPERIMENTS}\label{sec:experiments}

In this section we will present two data distributions and we will show how the implied guarantees from Section~\ref{sec:mainresults} compare to the optimum.

\subsection{Concentric spheres}\label{subsec:advspheres}

First let's analyze the concentric spheres dataset considered in \cite{adversarialSpheres}. The data distribution consists of two concentric spheres in $d$ dimensions: we generate $x \in \mathbb{R}^d$ where $||x||_2$ is either $1.0$ or $1.3$, with equal probability assigned to each norm. We associate with each $x$ a label $y$ such that $y = -1$ if $||x||_2 = 1.0$ and $y = +1$ otherwise. 

First observe that the data is perfectly separable and that the optimal classifier
$$
    g_{opt}(x) =
    \begin{cases}
      -1, & \text{if}\ ||x||_2 \leq 1.15 \\
      +1, & \text{otherwise}
    \end{cases}
$$
obtains $AR(g_{opt},0.15) = 0$, which is the information-theoretic optimum. Assume that we have access to a classifier $f$ such that $R(f) = \delta$. Now we want to analyze the performance of our algorithm. More precisely, we compare our algorithm to the set of classifiers
$$\mathcal{H} := \{g : \R^d \xrightarrow{} \{-1,1\} \ | \ R(g) \geq \delta\} \text{,}$$
and not $g_{opt}$. The constraint $R(g) \geq \delta$ is natural as it means that we want to be competitive against  classifiers that are no better than the input classifier $f$.

Now assume that we want to produce a classifier $ALG(f)$ such that $AR(ALG(f),\e) \leq \eta$, for some $\eta \in \R_+$. We should compare the following two quantities:
\begin{equation}\label{eq:ealgdef}
\e_{alg} := \argmax_{\e \in \R_+} \ [AR(ALG(f), \e) \leq \eta ] \text{,}
\end{equation}
\begin{equation}\label{eoptdef}
\e_{opt} := \argmax_{\e \in \R_+} \left[ \min_{g \in \mathcal{H} } AR(g,\e) \leq \eta \right] \ \text{.}
\end{equation}

Observe that the separation function for this dataset is\footnote{The separating function $S(\epsilon)$ does not reach $1$ for large values of $\epsilon$ as one might think at first since one can always completely remove one class in order to guarantee a separation of $\infty$.}
$$
    S(\e) =
    \begin{cases}
      0, & \text{if}\ \e < 0.3, \\
      1/2, & \text{otherwise}.
    \end{cases}
$$
Then Theorem~\ref{thm:balls} guarantees that we can produce $ALG(f)$ so that: 
$$
AR(ALG(f), \epsilon) \leq 2\delta + O(\epsilon \cdot d)  \text{.}
$$
Using definition \eqref{eq:ealgdef} this gives us that $\e_{alg} \geq  \Theta(\frac{\eta - 2\delta}{d})$.

Now let $g \in \mathcal{H}$. Recall that by definition $R(g) \geq \delta$.  Let $S_{in}$ and $S_{out}$ denote the inner and outer sphere, respectively. 
Assume that $E$ and $E'$ are the sets of misclassified points on the inner and outer sphere respectively. Without loss of generality we may assume that $\mu(E) \geq \delta/2$ ($\mu$ is the measure corresponding to $\mathcal{D}$). Notice that for all $\e$ we have $AR(g,\e) \geq \mu(E + B_{\e})$. Moreover, the isoperimetric inequality for spheres states that among all sets of measure $\delta/2$ the one that minimizes $\mu(E + B_{\e})$ is a spherical cap of this volume, see \cite{adversarialSpheres}. Let's call this cap $C$. Now observe that $\mu(C + B_{\e}) \approx \frac{\delta}{2}(1+\e)^d$. This means that $\e_{opt} \leq O \left(\frac{\log(\delta/\eta)}{d} \right)$. 

Combining lower and upper bounds we get that:
\begin{equation}\label{eq:compratio}
\frac{\epsilon_{opt} }{ \epsilon_{alg} } \leq O \left( \frac{ \log(\eta/\delta) }{ \eta - 2\delta } \right) \text{.}
\end{equation}

That is, our method achieves the target adversarial risk but for perturbations that are $O \left( \frac{ \log(\eta/\delta) }{ \eta - 2\delta } \right)$  smaller than the optimum. For example in a regime where $\log(\delta/ \eta)$ remains smaller than a constant we get a Markov-style tradeoff between the target adversarial risk $\eta$ and the optimality of $\e$.

It was shown in \cite{adversarialSpheres} that neural networks trained on concentric spheres dataset achieve very small risk. When one of the trained networks was evaluated on $20$ million samples no errors were observed. This means that $R(f)$ for the base classifier $f$ might be really small for this dataset. If for the target adversarial risk we have $\eta >> R(f)$ then the bound \eqref{eq:compratio} might not be satisfactory. It is an interesting research direction to analyze the regime where $\eta >> R(f)$.

\subsection{Intersecting circles}

Let $u_1,u_2$ be a pair of orthonormal vectors in $\mathbb{R}^d$. Let $C_{-1}, C_{+1} \subseteq \mathbb{R}^d$ be two circles in the $2$-dimensional subspace spanned by $u_1,u_2$ of radius $1$ centered at $0$ and $u_1$ respectively. The data distribution is defined as follows: we generate $x \sim U(C_{-1} \cup C_{+1})$ and we associate with each $x$ a label $y$ such that $y = -1$ if $x \in C_{-1}$ and $y = +1$ otherwise.

Note that for $\e \leq 1/10$, $S(\e) = \Theta(\e)$. This is true since in order to $\e$-separate the classes we need to remove the points close to the two intersection points. Note that $\text{supp}(\mathcal{D})$ is a union of two $1$-dimensional manifolds whose second fundamental form is bounded by $\Theta(1)$   (Theorem~\ref{thm:manifold} also works in this case). Hence, using Theorem~\ref{thm:manifold} for all $\e < 1/10$ and $\alpha > 0$:
$$AR(g,\alpha \e) \leq O(AR(\e) + R(f) + \alpha)\text{.}$$
This means that if $\alpha$ is a small constant and $R(f)$ is small then $g$ is only a constant times (plus an additive error) worse than the optimal classifier for adversarial perturbations which are only $\alpha$ times smaller. Note that the final guarantee does not depend on the dimension of the ambient space but only on the dimension of the manifolds themselves, which in this case is $1$.

\section{OPEN PROBLEMS \& RESEARCH DIRECTIONS}

One important open problem is to consider improvements of Theorem~\ref{thm:balls}. In this theorem the guaranteed robustness radius degrades with the dimensionality $d$ of the space. One might hope to get a better dependence on $d$. In some regimes however it might be hard to achieve an improvement as discussed in Subsection~\ref{subsec:advspheres} (see competitive guarantee \eqref{eq:compratio}).

It is also interesting to analyze different threat models. Imagine that we want the classifier to be robust against an \textbf{oblivious} adversary, that is an adversary that has access to $f$ and the algorithm's code but does not know the randomness used by the algorithm. In Appendix~\ref{sec:oblivious} we show that in this model it's possible to achieve the bound 
$$AR(g,\epsilon) \leq 2S \left(\frac{ \bf{\sqrt{d}} \cdot \e}{\alpha} \right) + 2R(f) + O(\alpha) \text{.}$$
Note that the main difference compared to Theorem~\ref{thm:balls} is that we have the factor $\bf{\sqrt{d}}$ instead of $\bf{d}$. Intuitively this means that we are be able to get the same adversarial risk for perturbations that are $\sqrt{d}$ bigger.


Another research direction is to improve the running time of the algorithms so they become more practical, especially the ones using Ball carving partition. These methods suffer from the high dimension of the ambient space $\mathbb{R}^d$, but as discussed in Subsection C.0.2 
there might be hope to improve the runtime per query to $2^{O(dd(\supp(\mathcal{D}),\e))}$. This would be a significant improvement for low-dimensional data distributions. 

Finally we can look at \textit{randomized smoothing} and the algorithm presented in this paper as two ends of a spectrum. The former is fast but doesn't guarantee good adversarial risk. The latter is slower but produces a robust classifier. One might hope to find a smooth tradeoff between the runtime and the adversarial risk guarantee.

\bibliography{sample_paper.bib}

\begin{appendix}
\section{Hard distribution for randomized smoothing described in \eqref{eq:restrictedsmoothing}}
\label{sec:harddistribution}

Consider the following data distribution. For $\e$ that will be fixed later, let $S_{\e}(0) \subseteq \mathbb{R}^d$ be a sphere of radius $\e$ around $0$ and $N \subseteq S_{\e}(0)$ be a set of cardinality $e^{0.118d}$ such that for all $x, y \in N, x \neq y$ we have $||x-y||_2 \geq 1.2\e$. One can show that such a set exists using bounds for the surface area of spherical caps in high dimension (see \cite{Blum15foundationsof}). 

Let the binary classification task be as follows. Let the distribution $\mathcal{D}_{+1}$ for class $+1$  be such that $\text{supp}(\mathcal{D}_{+1}) = (N \cup \{0\}) + B_{0.01\e}$ (where the $+$ denotes the Minkowski sum). The density function on $B_{0.01\e}(0)$ is $e^{0.108d}$ times larger than the one on $B_{0.01\e}(u)$ for every $u \in N$. Now let $\mathcal{D}_{-1}$ be such that $\text{supp}(\mathcal{D}_{-1}) \cap \text{supp}(\mathcal{D}_{+1}) = \emptyset$ and each class has probability $1/2$. Now assume that the points that classifier $f$ misclassifies are exactly points in $B_{0.01\e}(0)$. Then the standard error of $f$ is at most $e^{-0.01d}$. Now let $\e := \sqrt{(d\sigma)/10}$. One can verify that when $g$ is computed according to \eqref{eq:regularsmoothing} then for all $x \in N + B_{0.01\e}$ we have $g(x) = -1$, which means that $g$ misclassifies all points from $N + B_{0.01\e}$. So the standard error of $g$ is at least $25\%$. This means that the error of $g$ is $e^{\Theta(d)}$ times larger than the error of $f$!

\begin{remark}
One might argue that this example was crafted artificially and that in the ``real world'' we can choose $\sigma$ depending on the data. However it is possible to construct examples such that for {\em any} reasonable choice of $\sigma$ a dynamic similar to the one presented above occurs. The idea is to put a collection of the above configurations at different scales and far from each other.
\end{remark}
\section{Generalization of definitions to nonseparable learning tasks}
\label{sec:generalizationofdefs}

\begin{customdef}{2a}
For a binary classification task and a classifier $f : \mathbb{R}^d \xrightarrow{} \{-1, 1\}$ we define \textbf{R}isk as
$$R(f) := \int p_X(x) \sum_{y \in \{-1, 1\}} \mathbb{P}_{Y|X}(y|x) \mathbbm{1}_{\{f(x)=y\}} dx \text{.}$$
\end{customdef}

\begin{customdef}{3a}
For a binary classification task, a classifier $f : \mathbb{R}^d \xrightarrow{} \{-1, 1\}$, and $\epsilon  \geq 0$ we define \textbf{A}dversarial \textbf{R}isk as
$$AR(f,\epsilon) := \int p_X(x) g(f, x, \epsilon) dx \text{,}$$
where
$$g(f,x, \epsilon) := 
\begin{cases}
\mathbb{P}_{Y| X}(-1 \mid x), & B_{\epsilon}(x) \subseteq M_1(f), \\
\mathbb{P}_{Y|X}(1 \mid x), & B_{\epsilon}(x) \subseteq M_{-1}(f), \\
1, & \text{otherwise}, \\
\end{cases}$$
where $M_{y} = f^{-1}(\{y\})$, $y \in \{-1, 1\}$.
We also introduce the notation:
$$AR(\epsilon) := \inf_{f} AR(f,\epsilon) \text{,}$$
to denote the optimal classification error for that classification task with a given $\epsilon$.

Note that this definition assumes that the adversary, apart from $x$, has also access to the label $y$. In other words, we prove bounds with respect to a strong adversary.
\end{customdef}

\begin{customdef}{4a}[\textbf{Separation function}]
For a binary classification task we define the separation function $S(\epsilon)$ as follows:
$$S(\epsilon) := \!\!\!\!\!\!\!\!\!\!\!\!\!\!\!\!\inf_{\stackrel{E_{-1}, E_1 \subseteq \mathbb{R}^d}{ d(\mathbb{R}^d \setminus E_{-1}, \mathbb{R}^d \setminus E_1) \geq \epsilon}} \sum_{y \in \{-1, 1\}} \int_{x \in E_y} p_X(x) \mathbb{P}_{Y | X} (y \mid x) dx  \text{.}$$
For a given $\epsilon>0$ this function returns the minimum probability mass that needs to be removed so that the classes are separated by an $\epsilon$-margin.
\end{customdef}

\begin{restatable}{lemma}{perfectclassifiernonsep}
\label{lem:perfectclassifiernonsep}
For all  binary classification tasks and all $\epsilon \geq 0$ we have that:
$$AR(\epsilon) = S(2\epsilon) \text{.}$$
\end{restatable}

\begin{proof}
First we prove that $AR(\epsilon) \leq S(2\epsilon)$. Let $E_{-1}$ and $E_1$ be the minimizer sets from the definition of $S(2\epsilon)$.  Let $f(x) := -1$ if $d(x, \mathbb{R}^d \setminus E_{-1}) \leq \epsilon$ and $f(x) := 1$ otherwise. Then observe that for all $x \in (\mathbb{R}^d \setminus E_{-1})$, $B_{\epsilon}(x) \subseteq M_{-1}(f)$ and for all $x \in (\mathbb{R}^d \setminus E_1)$, $B_{\epsilon}(x) \subseteq M_1(f)$. Hence $AR(\epsilon) \leq S(2\epsilon)$.

Now we prove that $AR(\epsilon) \geq S(2\epsilon)$. Let $f$ be a classifier with $AR(f,\epsilon) = r$. Let $E_{-1}$ be the set of all points $x \in \mathbb{R}^d$ so that $B_{\epsilon}(x) \not \subseteq M_{-1}(f)$ and let 
$E_1$ be the set of all points $x \in \mathbb{R}^d$ so that $B_{\epsilon}(x) \not \subseteq M_1(f)$. It follows that
\begin{align}
d(\mathbb{R}^d \setminus E_{-1}, \mathbb{R}^d \setminus E_1) \geq 2\epsilon. \label{equ:distance}
\end{align}
But now note that for this choice of sets $E_{-1}$ and $E_1$, 
$$\sum_{y \in \{-1, 1\}} \int_{x \in E_y} p_X(x) \mathbb{P}_{Y | X} (y \mid x) dx=r=AR(f,\epsilon).$$
Hence, for $S(2\epsilon)$, being defined as the infimum over all choices of sets $E_{-1}$ and $E_1$ which fulfill (\ref{equ:distance}), we have
$S(2\epsilon) \leq r = AR(f,\epsilon)$.
\end{proof}
\section{Running time discussion}
\label{sec:runningtime}

Let us now analyze the running times of Algorithm~\ref{alg:framework} as a function of the used partition as well as the method of estimating $g$. In the stated bounds we will assume that each evaluation of $f$ takes time $t$.  

\subsubsection{Cube partition}

\textbf{Scheme B:} First let's analyze the performance of Cube partitions together with assumption \eqref{eq:approxtog}. To evaluate $\hat{g}(x)$ we need to locate a cube to which $x$ belongs to and smooth $f$ over that cube. Smoothing is approximated by a sample mean and as argued before $O(\log(Q))$ samples suffice. So in the end the running time per query is $O(t \cdot \log(Q))$. If we store (using hashing techniques) previous function evaluations then the query time drops to $O(1)$ for queries from cubes that were already queried before.

\textbf{Scheme A:} If we use \eqref{eq:approxtogdata} instead of \eqref{eq:approxtog} then we first perform a preprocessing step in which we sample a set $U$ of unlabeled samples of size (\ref{equ:numofsamples}). Then using standard hashing techniques we can create a data structure of size (\ref{equ:numofsamples}) that for a point $x \in \mathbb{R}^d$ will provide access to $U \cap \pi(x)$ in $O(1)$ time per accessed element. Having that query time is $O(t \cdot \log(Q))$ because, as argued before, for each cube it's enough to consider only that many samples to compute a good estimator. Similarly as in the previous case for repeated queries time drops to $O(1)$.

\subsubsection{Ball carving partition}\label{subsec:runtimeballs}

\textbf{Scheme B:} Now let's analyze Ball carving partitions with assumption \eqref{eq:approxtog}. The situation here is much more complicated and the implementation is much more involved. To compute $g$ we need access to an $\e/4$-net $N$ that covers $\supp(\mathcal{D})$. We create $N$ on the fly. I.e., we start with $N = \emptyset$ and when a query $q \in \supp(\mathcal{D})$ arrives then if $q \not\in \bigcup_{u \in N} B_{\e/4}(u)$ we add $q$ to $N$. Whenever we add a vertex to $N$ we sample a new permutation $\sigma$ on $N$, which corresponds to a new partition. This means that when a point is added to $N$ then $g$ can change. But once the construction process stabilizes then $g$ remains fixed. Using Chebyshev's inequality one can verify that if for $O(1/R(f))$ consecutive queries we don't add new vertices to $N$ then $\bigcup_{u \in N} B_{\e/4}(u)$ contains $1 - O(R(f))$ probability mass of $\mathcal{D}$ with probability $1 - R(f)$. When this event occurs we can stop changing $N$ as the probability mass not covered by $N$ is $O(R(f))$ with high probability. Finally observe that:
\begin{align*}
|N| 
&\leq \max_{\stackrel{N' \subseteq \supp(\mathcal{D}) :}{ N' \text{ is } \e/4-net}} |N'| \\
&\leq \min_{\stackrel{N' \subseteq \supp(\mathcal{D}) :}{ N' \text{eq is } \e/8-net}} |N'| =: Q_{max} \text{.}
\end{align*}

Now let's analyze the running time. Consider a query $q \in \mathbb{R}^d$. To compute $g(q)$ we must first check if $q$ should be added to $N$ and this can be done in $O(|N|)$ time. Then we choose a random permutation and locate the set $\pi(q)$ to which $q$ belongs (also in $O(|N|)$ time).


After locating $u \in N$ such that $q \in B_{R}(u) \setminus  \bigcup_{w: \sigma(w) < \sigma(u)} B_{R}(w) = \pi(q)$ we need to sample points uniformly at random from $\pi(q)$ to compute sample mean to estimate $g(q)$. One way to do that is to use Hit-and-Run sampling. To generate a uniformly random point from $\pi(q)$ we generate a sequence $\{x_i\} \subseteq \pi(q)$ according to the following rule: 
\begin{itemize}
    \item $x_0 = q$,
    \item to generate $x_{i+1}$ from $x_i$ we first pick a random direction $v$. We find minimal and maximal values such that $x_i + \theta \cdot v \in \pi(q)$. We pick $\theta^*$ uniformly from the interval $[\theta_{\min}, \theta_{\max}]$ and we set $x_{i+1} := x_i + \theta^* \cdot v$.
\end{itemize}
After generating some number of points we declare the last point as a point drawn from $U(\pi(q))$. The time needed to generate one sample is $k \cdot O(|N|)$, where $k$ is the number of iterations we perform. 

To get an algorithm with a theoretical guarantee on the running time for sampling points one can resort to an algorithm from \cite{approxpolytopevolume}. That algorithm implicitly, in polynomial in $d$ time, samples a point uniformly at random from a convex body. It is possible to adapt the algorithm to the case of non-convex bodies (as our set $\pi(q)$ is not necessarily convex). We can think that $\pi(q)$ is "close" to being convex as it is defined by a carving process with balls of equal radii. Recall from previous discussion that it's enough to have $O(\log(Q))$ samples per set. So in the end if we use this algorithm then the running time for computing $g(p)$ will be $O\left(\poly(d) \cdot \log(Q) \cdot |N| + t \log(Q)\right) = O\left(\poly(d) \cdot Q_{\max} \log(Q_{\max}) + t \log(Q_{\max})\right)$.

\textbf{Scheme A:} If we use \eqref{eq:approxtogdata} instead of \eqref{eq:approxtog} then we first perform a preprocessing step in which we sample a set $U$ of unlabeled samples of size \eqref{equ:numofsamples} (with $Q$ set to $Q_{\max}$). Then we use a greedy algorithm to find a maximal subset $N \subseteq U$ such that for every $u,w\in N, u \neq w$ we have $||u-w||_2 \geq \e/4$. Using Chebyshev's inequality one can argue that with high probability $\bigcup_{u \in N} B_{\e/4}(u)$ contains $1 - O(R(f))$ mass of $\mathcal{D}$. We then perform the ball carving partition using $N$. Then in time $\widetilde{O}(Q_{\max}^2)$ we create a data structure of size \eqref{equ:numofsamples} that for a point $u \in N$ will provide access to $U \cap (B_{R}(u) \setminus  \bigcup_{w: \sigma(w) < \sigma(u)} B_{R}(w))$ in $O(1)$ time per accessed element. Then for a query $q$ we need to first locate $u \in N$ such that $q \in \pi(u)$, which takes $O(Q_{\max})$ time and then we compute sample mean in $O(t \cdot \log(Q_{\max}))$ time. So in the end the running time per query is $O(t \cdot \log(Q_{\max}))$. If there is a repeated query for the same set then we can answer it in $O(Q_{\max})$ time.

The $O(Q_{\max})$ factor in both approaches is far from perfect. However there might be hope to decreasing this factor to $2^{O(dd(\text{supp}(M),\e))}$ using locality sensitive hashing techniques (\cite{indyklsh}) as in principle we only need to check points in the neighborhood of $q$ to determine $\pi(q)$ and in this neighborhood we have only $2^{O(dd(\text{supp}(M),\e))}$ of them. It might also be possible to reduce the running time further which might be an interesting research direction.

\begin{remark}
Assume that the data is supported on a lower dimensional manifold of dimension $d'$ and satisfies the assumptions from Theorem~\ref{thm:manifold}. Then robustness guarantees of our algorithms improve automatically with $d'$. That is we don't need to provide $d'$ as the input to our algorithms. 
\end{remark}
\section{Omitted proofs}
\label{sec:proofs}

\subsection{Proofs of Section~\ref{sec:preliminaries}}

\bestclassifier*

\begin{proof}
First we prove that $AR(\epsilon) \leq S(2\epsilon)$. Let $E$ be the minimizer set from the definition of $S(2\epsilon)$.  Let $f(x) := -1$ if $d(x, M_{-} \setminus E) \leq \epsilon$ and $f(x) := +1$ otherwise. Then observe that for all $x \in (M_{-} \setminus E) \cup (M_{+} \setminus E)$ there does not exist an $\eta$ so that $f(x+\eta) \neq h(x)$. Hence $AR(\epsilon) \leq S(2\epsilon)$.

Now we prove that $AR(\epsilon) \geq S(2\epsilon)$. Let $f$ be a classifier with $AR(f,\epsilon) = r$. That means that there exists $A \subseteq \mathbb{R}^d$ such that 
\begin{itemize}
    \item $\mathbb{P}_X(X \in A) \geq 1 - r$,
    \item for all $x \in A$ we have $\forall \ \eta \in B_{\epsilon} \ f(x+\eta) = h(x)$.
\end{itemize}
This means that $\mathbb{R}^d \setminus A$ is a $2\epsilon$-separator for that binary task, so in turn $S(2\epsilon) \leq r = AR(f,\epsilon)$.
\end{proof}


\subsection{Proofs of Section~\ref{sec:doublingdimension}}

\doublingdimofrd*

\begin{proof}
Let $B_{\e}(0) \subseteq \mathbb{R}^d$ be a ball of radius $\e$ for some $\e > 0$. Let $N$ be an $\e/2$-net of $B_{\e}(0)$. Notice that all balls in $\{ B_{\e/4}(u) : u \in N\}$ are pairwise disjoint and that $\bigcup_{u \in N} B_{\e/4}(u) \subseteq B_{5\e/4}(0)$. Hence $|N| \leq \frac{\text{vol}(B_{5\e/4})}{\text{vol}(B_{e/4})} = 5^d$.
\end{proof}

\notmanypoints*

\begin{proof}
As $t \leq \e$ we can use the definition of $\e$-doubling dimension and get that $B_t(x)$ can be covered with $2^{dd}$ balls of radius $t/2$. Iterating that argument, we conclude that $B_t(x)$ can be covered by $2^{dd \lceil \log \frac{2t}{r} \rceil}$ balls of radius $r/2$. But every such ball can contain at most one point from $N$ so $|B_t(x) \cap N|$ is also upper bounded by $2^{dd \lceil \log \frac{2t}{r} \rceil}$.
\end{proof}

\subsection{Proofs of Section~\ref{sec:randompartitions}}

\propertyofpadded*

\begin{proof}
\begin{align*}
&\mathbb{E}_{\Pi \sim \mathcal{P}}[\mathbb{P}_{X \sim \mathcal{D}}[B_{\e/\beta}(X) \not\subseteq \Pi(X)]] \\
& =\mathbb{E}_{\Pi \sim \mathcal{P}}[\mathbb{E}_{X \sim \mathcal{D}}[\mathbbm{1}_{\{B_{\e/\beta}(X) \not\subseteq \Pi(X) \}}]] \\
& =\mathbb{E}_{X \sim \mathcal{D}}[\mathbb{E}_{\Pi \sim \mathcal{P}}[\mathbbm{1}_{\{B_{\e/\beta}(X) \not\subseteq \Pi(X) \}}]] \\
&= \mathbb{E}_{X \sim \mathcal{D}} \left[\mathbb{P}_{\Pi \sim {\mathcal P}}[B_{\e/\beta}(X) \not\subseteq \Pi(X)] \right]  \leq \delta.
\end{align*}
\end{proof}

\cubepartispadded*

\begin{proof}
For all $x \in \mathbb{R}^d$, $\text{diam}(\Pi(x)) = \epsilon$ by construction. Let $A=\left[0,\frac{\e}{\sqrt{d}}\right]^d$. This is the set of {\em all} points of one fundamental cube. Let $G = \left[\frac{\e}{\beta}, \frac{\e}{\sqrt{d}} - \frac{\e}{\beta} \right]^d$ and note that $d \left(G,\mathbb{R}^d \setminus A \right) = \frac{\e}{\beta}$. $G$ represents the set of all {\em good} points inside $A$, in the sense that if we center a sphere of radius $\epsilon/\beta$ at one of those points the whole sphere stays contained inside $A$.  Now observe that 
\begin{equation}\label{eq:problwrbnd}
\frac{\text{vol}(G)}{ \text{vol} \left(A \right)} = \left(1 - \frac{2\sqrt{d}}{\beta} \right)^d \geq 1 - \frac{2 \cdot d^{1.5}}{\beta} \text{.}
\end{equation}
Let $v$ be the shift that generates the partition $\pi$. Consider the set $I(v) := \bigcup_{z \in v + \frac{\e}{\sqrt{d}}\cdot \mathbb{Z}^d} (G + z)$. Using \eqref{eq:problwrbnd}, we conclude by noting that for every $x \in \mathbb{R}^d$
$$\mathbb{P}_{\Pi \sim {\mathcal P}}[B_{\frac{\e}{\beta}}(x) \not\subseteq \Pi(x)] \leq \mathbb{P}_{V \sim U \left(A \right)}[x \not\in I(V)] \leq \frac{2 d^{\frac{3}{2}}}{\beta} \text{.}$$
\end{proof}

\subsection{Proofs of Section~\ref{sec:partitions}}

\boundriskofg*

\begin{proof}
Let us first prove the weaker bound $R(g) \leq 3S(\e) + 2R(f) \text{.}$
Let $E$ be the minimizer set from the definition of $S(\epsilon)$ and $M_{-} = h^{-1}(\{-1\}),M_{+} = h^{-1}(\{1\})$. Then we know that $d(M_{-} \setminus E, M_+ \setminus E) \geq \epsilon$ and $\mathbb{P}_{X \sim \mathcal{D}}(X \in E) \leq S(\epsilon)$. Let $Q \subseteq M_- \cup M_+$ be the set of missclassified points of $f$ in $M_- \cup M_+$. Observe that
\begin{align*}
R(g) 
& \leq S(\epsilon) + \sum_{u \in N, \hat{\Pi}(u) \cap M_- \neq \emptyset, g(\hat{\Pi}(u)) = +1} \mu(\hat{\Pi}(u))  \\
& \phantom{\leq S(\epsilon) \;} + \sum_{u \in N, \hat{\Pi}(u) \cap M_+ \neq \emptyset, g(\hat{\Pi}(u)) = -1} \mu(\hat{\Pi}(u)) \\
&\leq S(\epsilon) + \sum_{\stackrel{u \in N, \hat{\Pi}(u) \cap M_- \neq \emptyset,}{ g(\hat{\Pi}(u)) = +1}} 2\mu(\hat{\Pi}(u) \cap (Q \cup E)) \\
& \phantom{\leq S(\epsilon) \;} + \sum_{\stackrel{u \in N, \hat{\Pi}(u) \cap M_+ \neq \emptyset,}{ g(\hat{\Pi}(u)) = -1}} 2\mu(\hat{\Pi}(u) \cap (Q \cup E)) \\
&\leq S(\epsilon) + 2(\mu(Q) + \mu(E)) \\
&\leq 3 S(\epsilon) + 2R(f) \text{.}
\end{align*}
To see that the claimed stronger bound is valid note the following. Every point in $E$ will appear either in exactly one of the two sums or it will be counted by the term $S(E)$. In the first two cases it is weighted by a factor $2$ and in the second case it is weighted by a factor $1$. This gives rise to the term $3 S(E)$. But no point of $E$ appears in both of those cases. We can therefore tighten this term to $2 S(E)$.
\end{proof}

\combineeverything*

\begin{proof}
We will prove that Algorithm~\ref{alg:framework} invoked with $f$ and $\Pi \sim \mathcal{P}$ satisfies the statement of the Lemma. By Fact~\ref{fact:basicsaboutrisks}
\begin{equation}\label{eq:advriskbnd}
AR(g,\e) \!\leq \!R(g) \!+\! \mathbb{P}_{X \sim \mathcal{D}}[g \text{ $\neg$ constant on } B_\e(X)].
\end{equation}
By Lemma~\ref{lem:boundriskofg} we have:
\begin{equation}\label{eq:upperbndforrisk}
R(g) \leq 2S(\e\beta) + 2R(f) \text{.}
\end{equation}
Moreover, by Corollary~\ref{cor:propertyofpadded} we have that: 
\begin{equation}\label{eq:pointsclosetobndr}
\mathbb{E}_{\Pi \sim \mathcal{P}}[\mathbb{P}_{X \sim \mathcal{D}}[B_{\e}(X) \not\subseteq \Pi(X)]] \leq \delta.
\end{equation}
But we also know from the definition of $g$ that
\begin{align}
&\mathbb{P}_{X \sim \mathcal{D}}[g \text{ is not constant on } B_\e(X)] \leq  \nonumber \\ &\mathbb{P}_{X \sim \mathcal{D}}[B_{\e}(X) \not\subseteq \Pi(X)] \text{.} \label{eq:whenfunctionchanges}
\end{align}
Combining \eqref{eq:advriskbnd},\eqref{eq:upperbndforrisk},\eqref{eq:pointsclosetobndr} and \eqref{eq:whenfunctionchanges} we get that in expectation over the random choices of the algorithm
\begin{align*}
AR(g,\e) &\leq 2S(\e\beta) + 2R(f) + \delta \\
&= \frac{2S(\e\beta)}{S(2\e)} AR(\e) + 2R(f) + \delta \text{,}
\end{align*}
where in the last equality we used Lemma~\ref{lem:bestclassifier}. Note that the last inequality is only valid if $AR(\e) > 0$.
\end{proof}
\section{Oblivious adversary}\label{sec:oblivious}

Let's consider the model where the adversary has full knowledge of the base classifier $f$ and the code of the algorithm $ALG$ that produces $g$ but doesn't have access to random bits used by $ALG$. Then the following is true:

\begin{theorem}\label{thm:lipshitzpart}
For every separable binary classification task in $\R^d$ and for every $\e \in \R_+$ there exists a randomized algorithm $ALG$ that, given black-box access to $f : \R^d \xrightarrow{} \{-1,1\}$, provides query access to a function $g : \R^d  \xrightarrow{} \{-1,1\}$ such that:
\begin{itemize}
    \item $R(g) \leq 2 S(\e) + 2 R(f)$,
    \item For every $x,x' \in \R^d$ we have that:
    $$\mathbb{P}_{ALG}[g(x) \neq g(x')] \leq O\left(\frac{\|x-x'\|_2 \cdot \sqrt{d}}{\e}\right) \text{.}$$
\end{itemize}
\end{theorem}

\begin{proof}
The proof of this theorem is an adaptation of a random partition technique from \cite{sqrtdpartitions}. This paper presents an algorithm that creates a random partition that is $(\e,O(\sqrt{d}))-Lipshitz$ (a notion similar to \textit{padded} partitions), that is a random partition that is \textit{$\e$-bounded} and for every $x,x' \in \R^d$:
$$\mathbb{P}[\Pi(x) \neq \Pi(x')] \leq O\left(\frac{\|x-x'\|_2 \cdot \sqrt{d}}{\e}\right) \text{.} $$ Using this partition $ALG$ creates $g$ using the framework from Algorithm~\ref{alg:framework}. One can verify that this $g$ satisfies the statements of the theorem.
\end{proof}

\begin{remark}
We note that the Algorithm from \cite{sqrtdpartitions} is very similar to the random partition from Definition~\ref{def:ballspart} as it also performs a version of ball carving. Based on this similarity, it is tempting to conjecture that the ball carving partition  from Definition~\ref{def:ballspart} is $\left(\e,O(\sqrt{d})\right)-Lipshitz$ also. We leave this as an interesting open question. Moreover, we note that the Algorithm from \cite{sqrtdpartitions} can be easily adapted to any $\ell_p$ norm achieving $\left(\e, O(d^{1/2p})\right)-Lipshitz$ partition for $1 \leq p \leq 2$ and $\left(\e, O(d^{1-1/p})\right)-Lipshitz$ partition for $p > 2$. This means that using this technique one can get adversarial robustness guarantees for any $\ell_p$ norm for $p \geq 1$. 
\end{remark}

Now observe that Theorem~\ref{thm:lipshitzpart} gives us an algorithm $\mathcal{A}$ that is robust against any oblivious adversary. The algorithm works as follows: for a series of queries $x'_1,x'_2, \dots \in \R^d$ ($x'_i$'s are inputs crafted by the adversary), for every $i$, $\mathcal{A}$ using $ALG$ from Theorem~\ref{thm:lipshitzpart}, recomputes a new $g_i$ to answer query $x'_i$. We know that $R(g_i) \leq 2S(\e) + 2R(f)$ and moreover for every $x,x'$ we have $\mathbb{P}_{ALG}[g_i(x) \neq g_i(x')] \leq O\left(\frac{\|x-x'\|_2 \cdot \sqrt{d}}{\e}\right)$. This means that no matter what the strategy of the adversary is (this strategy might depend on $g_1(x'_1), \dots, g_{i-1}(x'_{i-1})$) the probability that the adversary will be able to construct two points such that $\|x_i - x'_i \|_2 \leq t$ and $g_i(x_i) \neq g_i(x'_i)$ is upper bounded by $O\left(\frac{t \cdot \sqrt{d}}{\e}\right)$. 

We summarize: For every $i$, if $X_i \sim \mathcal{D}$ at the $i$-th step and the adversary creates $X_i'$ such that $\| X_i - X_i' \|_2 \leq \e$ then for every $\alpha$:
\begin{align*}
&\mathbb{P}_{X_i,\mathcal{A}} (g_i(X_i') \neq h(X_i)) \leq \\ &2S \left(\frac{\sqrt{d} \cdot \e}{\alpha} \right) + 2R(f) + O(\alpha)\text{.}
\end{align*}
Observe the connection to Definition~\ref{def:adversarialrisk} which we restate here for convenience: 
$$AR(f,\epsilon) := \mathbb{P}_X (\exists \ \eta \in B_\epsilon \ f(X + \eta) \neq h(X)) \text{.}$$
The reason that we were able to gain a factor $\sqrt{d}$ in comparison to Theorem~\ref{thm:balls} is that we didn't need to ensure that a function is constant on a ball $B(x,\e)$. It was enough to show that it is constant for every fixed pair of nearby points as the adversary can only test one point at a time. 

This gain comes at a cost as we need to recompute the partition after every query. If one recomputes the partition every $k$ queries then by the union bound the guarantee changes to: 
\begin{align*}
&\mathbb{P}_{X_i,\mathcal{A}} (g_i(X_i') \neq h(X_i)) \leq \\ &2S \left(\frac{\sqrt{d} \cdot k \cdot \e}{\alpha} \right) + 2R(f) + O(\alpha)\text{.}
\end{align*}

\end{appendix}

\end{document}